\documentclass{article}

\usepackage[utf8]{inputenc} % allow utf-8 input
\usepackage[T1]{fontenc}    % use 8-bit T1 fonts
\usepackage{hyperref}       % hyperlinks

\usepackage{subfigure}
\usepackage{algorithmic}
\usepackage{algorithm}
\usepackage{graphicx}
\usepackage{tabularx} % tabularx
\usepackage{amsmath}
\usepackage{amssymb}
\usepackage{multirow,booktabs}
\usepackage{color}
\usepackage{multirow}
\usepackage{url}
\usepackage{amsthm}
\usepackage{rotating}
\usepackage{mathrsfs}
\usepackage{amsfonts}
\usepackage{amssymb}
\newtheorem{theorem}{Theorem}
\newtheorem{lemma}{\bf{Lemma}}
\newtheorem{definition}{\bf{Definition}}
\newtheorem{proposition}{\bf{Proposition}}

\title{Information Utilization Ratio in Heuristic Optimization Algorithms}
\author{Junzhi Li and Ying Tan\\
Computational Intelligence Laboratory\\
Peking University\\
\{ljz,ytan\}@pku.edu.cn}

\begin{document}

\maketitle

\begin{abstract}
Heuristic algorithms are able to optimize objective functions efficiently because they use intelligently the information about the objective functions. Thus, information utilization is critical to the performance of heuristics. However, the concept of information utilization has remained vague and abstract because there is no reliable metric to reflect the extent to which the information about the objective function is utilized by heuristic algorithms. In this paper, the metric of information utilization ratio (IUR) is defined, which is the ratio of the utilized information quantity over the acquired information quantity in the search process. The IUR proves to be well-defined. Several examples of typical heuristic algorithms are given to demonstrate the procedure of calculating the IUR. Empirical evidences on the correlation between the IUR and the performance of a heuristic are also provided. The IUR can be an index of how finely an algorithm is designed and guide the invention of new heuristics and the improvement of existing ones.
\end{abstract}

\section{Introduction}
In the field of computer science, many heuristic algorithms have been developed to solve complex non-convex optimization problems. Although optimal solutions are not guaranteed to be found, heuristics can often find acceptable solutions at affordable cost. The key to designing a heuristic algorithm is to use heuristic information about the objective function. Many algorithms \cite{srinivas1994genetic,storn1997differential,dorigo2006ant} are claimed to be reasonably designed because they use heuristic information intelligently. Even more algorithmic improvement works \cite{harik1999compact,zhang2009jade,7508443} are claimed to be significant because they use more heuristic information or use heuristic information more thoroughly than the original algorithms.

Empirically, heuristic information is used more thoroughly in more advanced algorithms. Suppose there are two search algorithms A and B for one dimensional optimization. Algorithm A compares the evaluation values of the solutions $x_1$ and $x_2$ to decide which direction (left or right) is more promising, while algorithm B uses their evaluation values to calculate both the direction and the step size for the next search. If the underlying distribution of objective functions is already known, then algorithm B is able to search faster than algorithm A if they are both reasonably designed because more information is utilized by algorithm B. It has been a common sense in the field of heuristic search that the extent of information utilization in a heuristic algorithm is crucial to its performance.

However, so far there is no reliable metric to reflect the extent of information utilization because unlike direct performance analyses \cite{Jones2001A,He2015Average}, this issue seems abstract. Especially, it is very difficult to measure how much information is used by an optimization algorithm.

In this paper, based on some basic concepts in the information theory, a formal definition of the information utilization ratio (IUR) is proposed, which is defined as the ratio of the utilized information quantity over the acquired information quantity in the search process. It is shown theoretically that IUR is well-defined. Examples of typical heuristic algorithms are also given to demonstrate the procedure of calculating IURs.

Theoretically, IUR itself is a useful index of how finely an algorithm is designed, but we still expect it to be practically serviceable, that is, we need to study the correlation between IUR and performance. However, the correlation between IUR and performance of heuristics is not so straightforward as some may expect. The performance of an optimization algorithm depends not only on the extent of information utilization but also on the manner of information utilization. Still, for algorithms that utilize information in similar manners, the influence of the IUR is often crucial, as is illustrated in the experiments.

After all, the metric of IUR helps researchers construct a clear (but not deterministic) relationship between the design and the performance of an optimization algorithm, which makes it possible that researchers can to some extent predict the performance of an algorithm even before running it. Thus, the IUR can be a useful index for guiding the design and the improvement of heuristic optimization algorithms.

%The remainder of this paper is structured as follows. In Section \ref{sec2} we give the definition of the IUR and prove it is defined. In Section \ref{sec3} we show how to calculate the IURs of practical optimization algorithms. In Section \ref{sec4} we discuss empirically the correlation between the IUR and the performance. Section \ref{sec5} concludes this paper.

\section{Information Utilization Ratio} \label{sec2}
\begin{definition}[Information Entropy]
The information entropy of a discrete random variable $X$ with possible values $x_i$ and probability density $p(x_i)$ is defined as follows.
\begin{equation}
H(X)=-\sum _{i} p(x_i)\log p(x_i).
\end{equation}
\end{definition}
\begin{definition}[Conditional Entropy]
The conditional entropy of two discrete random variables $X$ and $Y$ with possible values $x_i$ and $y_j$ respectively and joint probability density $p(x_i,y_j)$ is defined as follows.
\begin{equation}
H(X|Y)=-\sum _{i,j} p(x_i,y_j)\log \frac{p(x_i,y_j)}{p(y_j)}.
\end{equation}
\end{definition}
Some elementary properties of information entropy and conditional entropy are frequently used in this paper, which however cannot be present here due to the limitation of space. We refer readers who are unfamiliar with the information theory to the original paper \cite{6773067} or other tutorials. %According to our experience, reading the following part of this section may be quite challenging especially for readers who are either unfamiliar with the information theory or unfamiliar with optimization, though we have tried our best to explain our idea.

The following lemma defines a useful function for calculating the IURs of various algorithms.

\begin{lemma}
If $\eta_1,\eta_2, \dotsc, \eta_{g+1}\in \mathbb{R}$ are independent identically distributed random variables,
\begin{align}
H(I(\min(\eta_1,\eta_2, \dotsc, \eta_{g})<\eta_{g+1})) \notag \\
= - \frac{g}{{g + 1}}\log \frac{g}{{g + 1}}- \frac{1}{g + 1} \log \frac{1}{g + 1} \triangleq \pi(g),
\end{align}
where
$
I(x<y) = \left\{
\begin{array}{lll}
1       &      & \text{if } x<y\\
0       &      & \text{otherwise}
\end{array} \right.
$
is the indicator function.
\end{lemma}
$\pi(g)\in(0,1]$ is a monotonic decreasing function of $g$.

\begin{definition}[Objective Function]
The objective function is a mapping $f:\mathcal{X} \mapsto \mathcal{Y}$, where $\mathcal{Y}$ is a totally ordered set.
 %with $|\mathcal{Y}| = n$ and for any $x \in \mathcal{X}$, $f(x) \in \mathcal{Y}$ is independently and uniformly distributed.
\end{definition}

$\mathcal{X}$ is called the search space. The target of an optimization algorithm is to find a solution $x\in\mathcal{X}$ with the best evaluation value $f(x)\in \mathcal{Y}$.

\begin{definition}[Optimization Algorithm]
An optimization algorithm $\mathscr{A}$ is defined as follows.
\begin{algorithm}[H]
\caption{Optimization Algorithm $\mathscr{A}$}
\begin{algorithmic}[1]
\STATE $i\leftarrow0.$
\STATE $D_0\leftarrow\emptyset.$
\REPEAT
\STATE $i \leftarrow i + 1.$
\STATE Sample $X_i \in 2^\mathcal{X}$ with distribution $\mathscr{A}_i(D_{i-1})$.
\STATE Evaluate $f(X_i) = \{f(x)|x\in X_i\}$.
\STATE $D_{i} \leftarrow D_{i-1} \cup \bigcup_{x \in X_{i}} {\{ x,f(x)\} }.$
\UNTIL $i=g$.
%\RETURN $ \mathop {\arg\min }_{x \in D} f(x)$
\end{algorithmic}
\end{algorithm}
In each iteration, $\mathscr{A}_i$ is a mapping from $2^{\mathcal{X}\times \mathcal{Y}}$ to the set of all distributions over $2^\mathcal{X}$. $\mathscr{A}_1(D_0)$ is a pre-fixed distribution for sampling solutions in the first iteration. $g$ is the maximal iteration number.
\end{definition}
In each iteration, the input of the algorithm $D_{i-1}$ is the historical information, which is a subset of $\mathcal{X} \times \mathcal{Y}$, and the output $\mathscr{A}_i(D_{i-1})$ is a distribution over $2^\mathcal{X}$, with which the solutions to be evaluated next are drawn. Note that the output $\mathscr{A}_i(D_{i-1})$ is deterministic given $D_{i-1}$.

By randomizing the evaluation step (consider $y=f(x)$ as a random variable), we are able to investigate how much acquired information is used in an optimization algorithm. That is, to what extent the action of the algorithm will change when the acquired information changes. Review the example in the introduction. It is clear that the algorithm A only uses the information of ``which one is better", while the information of evaluation values are fully utilized by the algorithm B. But how to express such an observation? Any change in $y_1$ or $y_2$ would cause the algorithm B to search a different location, while only when $I(y_1>y_2)$ changes would the action of the algorithm A change. So, the quantity of utilized information can be expressed by the information entropy of an algorithm's action. The entropy of the action of the algorithm A is one bit, while the entropy of the action of the algorithm B is equal to the entropy of the evaluation values.
Assume $Z$ is the ``action" of the algorithm, $X$ is the positions of the solutions, $Y$ is the evaluation values, (they are all random variables), then we can roughly think the information utilization ratio is ${H(Z|X)}/{H(Y|X)}$. However, optimization algorithms are iterative processes, so the formal definition is more complicated.

\begin{definition}[Information Utilization Ratio]
If $\mathscr{A}$ is an optimization algorithm, the information utilization ratio of $\mathscr{A}$ is defined as follows.
\begin{equation}
\text{IUR}_\mathscr{A}(g) = \frac{\sum_{i = 1}^g {H(Z_i|\overline{X}_{i-1},\overline{Z}_{i-1})}} {\sum_{i = 1}^g {H(Y_i|\overline{X}_i,\overline{Y}_{i-1})}},
\end{equation}
where $g$ is the maximal iteration number,
$
X = \{X_1,X_2,\dotsc,X_g\}
$
is the set of all sets of evaluated solutions,
$
Y = \{f(X_1),f(X_2),\dotsc,f(X_g)\}
$
is the set of all sets of evaluation values,
$
Z = \{\mathscr{A}_1(D_0), \mathscr{A}_2(D_1),\dotsc, \mathscr{A}_g(D_{g-1})\}
$
is the output distributions in all iterations of algorithm $\mathscr{A}$, $\overline{X}_i \triangleq \{X_1,\dotsc,X_i\}, \overline{Y}_i \triangleq \{Y_1,\dotsc,Y_i\},\overline{Z}_i \triangleq \{Z_1,\dotsc,Z_i\}$, $\overline{X}_0 = \overline{Y}_0 = \overline{Z}_0 = \emptyset$.
\end{definition}

Fig. \ref{fig1} shows the relationship among these random variables. Generally, $X_i$ is acquired by sampling with the distribution $Z_i$, $Y_i$ is acquired by evaluating $X_i$, and $Z_i$ is determined by the algorithm according to the historical information $\overline{X}_{i-1}$ and $\overline{Y}_{i-1}$.

For deterministic algorithms (i.e., $H(X_i|Z_i)=0$), the numerator degenerates to $H(Z)$. If function evaluations are independent, the denominator degenerates to $\sum_{i = 1}^g {H(Y_i|X_i)}$.

\begin{figure}[htb]
  \centering
  \includegraphics[width=1.0\textwidth]{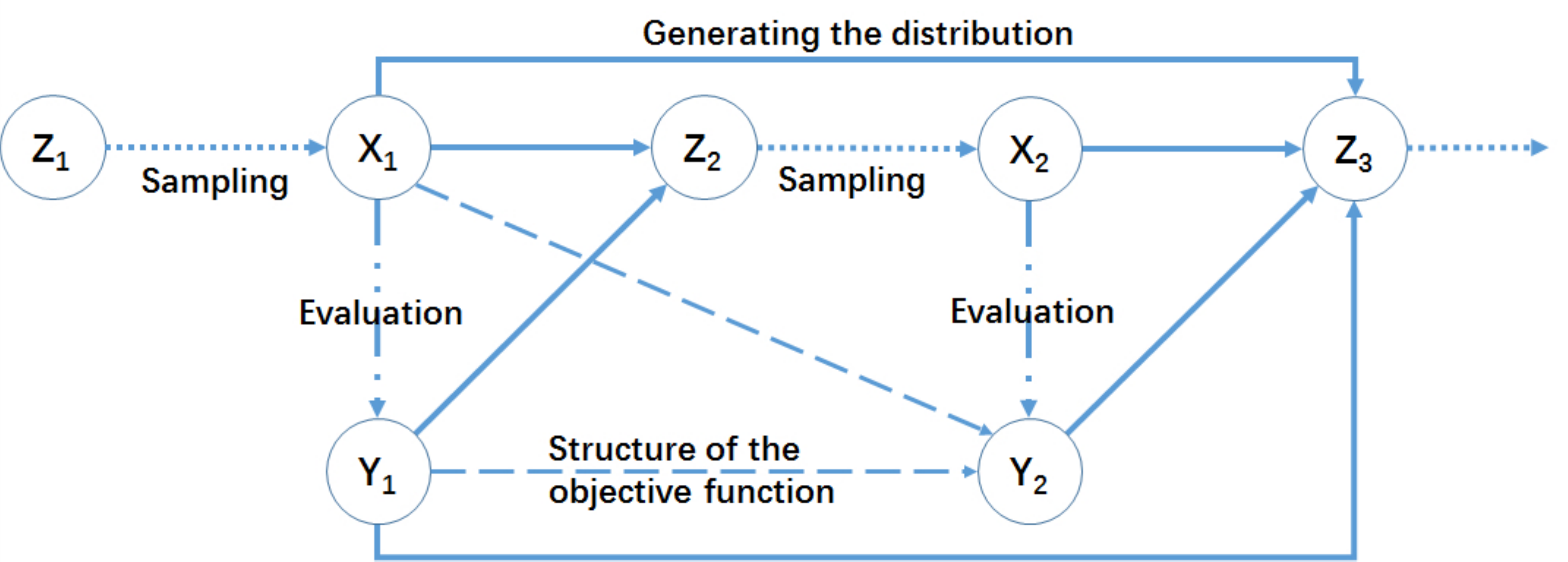}\\
  \caption{Graphic Model}
  \label{fig1}
\end{figure}

The following theorem guarantees that IUR is well defined.
\begin{theorem}
If $0<\sum_{i = 1}^g {H(Y_i|\overline{X}_i,\overline{Y}_{i-1})}<\infty$, then $0\le\text{IUR}_\mathscr{A}(g)\le1$.
\end{theorem}
\begin{proof}
\begin{align}
&H(X,Z) - \sum_{i = 1}^g {H(X_i|Z_i)} \\
&= \sum_{i = 1}^g {H(X_i,Z_i|\overline{X}_{i-1},\overline{Z}_{i-1})- \sum_{i = 1}^g H(X_i|\overline{X}_{i-1},\overline{Z}_i)} \\
&= \sum_{i = 1}^g {H(Z_i|\overline{X}_{i-1},\overline{Z}_{i-1})} \\\label{eq4}
&= \sum_{i = 2}^g {H(\overline{Z}_i|\overline{X}_{i-1})-\sum_{i = 2}^g H(\overline{Z}_{i-1}|\overline{X}_{i-1})} \\ \label{eq5}
&= \sum_{i = 2}^g {H(\overline{Z}_i|\overline{X}_{i-1})-\sum_{i = 2}^g H(\overline{Z}_i|\overline{X}_{i-1},\overline{Y}_{i-1})} - \sum_{i = 2}^g H(\overline{Z}_{i-1}|\overline{X}_{i-1})\notag \\
&+ \sum_{i = 2}^g H(\overline{Z}_{i-1}|\overline{X}_{i-1},\overline{Y}_{i-2}) \\
&= \sum_{i = 2}^g {-H(\overline{Y}_{i-1}|\overline{Z}_{i},\overline{X}_{i-1})+ \sum_{i = 2}^g H(\overline{Y}_{i-1}|\overline{X}_{i-1})} + \sum_{i = 2}^g H(\overline{Y}_{i-2}|\overline{Z}_{i-1},\overline{X}_{i-1})\notag \\
&- \sum_{i = 2}^g H(\overline{Y}_{i-2}|\overline{X}_{i-1}) \\\label{eq7}
&= \sum_{i = 2}^g -H(\overline{Y}_{i-1}|\overline{Z}_{i},\overline{X}_{i-1}) + \sum_{i = 2}^g H(\overline{Y}_{i-2}|\overline{Z}_{i-1},\overline{X}_{i-2}) + \sum_{i = 2}^g H(Y_{i-1}|\overline{X}_{i-1},\overline{Y}_{i-2}) \\ \label{eq8}
&= -H(\overline{Y}_{g-1}|Z,\overline{X}_{g-1}) + \sum_{i = 1}^g {H(Y_i|\overline{X}_i,\overline{Y}_{i-1})}  - H(Y_g|\overline{X}_g,\overline{Y}_{g-1}) \\
&\le \sum_{i = 1}^g {H(Y_i|\overline{X}_i,\overline{Y}_{i-1})}.
\end{align}
Eq. (\ref{eq4}) holds because \begin{equation}H(Z_1)=0.\end{equation}
Eq. (\ref{eq5}) holds because \begin{equation}H(\overline{Z}_i|\overline{X}_{i-1},\overline{Y}_{i-1}) = H(\overline{Z}_{i-1}|\overline{X}_{i-1},\overline{Y}_{i-2}) = 0.\end{equation}
Eq. (\ref{eq7}) holds because \begin{equation}\sum_{i = 2}^g H(\overline{Y}_{i-2}|\overline{Z}_{i-1},\overline{X}_{i-1})=\sum_{i = 2}^g H(\overline{Y}_{i-2}|\overline{Z}_{i-1},\overline{X}_{i-2}).\end{equation}
Eq. (\ref{eq8}) is by dislocation subtraction.
\end{proof}

The denominator in the definition $\sum_{i = 1}^g {H(Y_i|\overline{X}_i,\overline{Y}_{i-1})}$ represents the information quantity that is acquired in the search process. If function evaluations are independent, then $H(Y_i|\overline{X}_i,\overline{Y}_{i-1})= H(Y_i|X_i,\overline{X}_{i-1},\overline{Y}_{i-1}) = H(Y_i|X_i).$ While the numerator is more obscure. Actually it represents the quantity of the information about the objective function which is utilized by the algorithm (or in other words, the minimal information quantity that is needed to run the algorithm).
Firstly, $\sum_{i = 1}^g {H(Z_i|\overline{X}_{i-1},\overline{Z}_{i-1})} = \sum_{i = 1}^g H(Z_i|\overline{X}_{i-1},\overline{Z}_{i-1})-\sum_{i = 1}^gH(Z_i|\overline{X}_{i-1},\overline{Z}_{i-1},\overline{Y}_{i-1})$
is similar to the concept of information gain in classification problems \cite{Quinlan1986Induction}, which indicates the contribution of the information of $Y$ to the algorithm.
Secondly, the uncertainty of $X$ and $Z$ only lies in two aspects: the random sampling step and the lack of the information from $Y$. Thus $H(X,Z) - \sum_{i = 1}^g {H(X_i|Z_i)}$ can be regarded as the objective function's information that is utilized by the algorithm. And in fact, it is equal to the numerator.
Thirdly, the numerator equals the denominator minus $H(Y_g|\overline{X}_g,\overline{Y}_{g-1}) + H(\overline{Y}_{g-1}|Z,\overline{X}_{g-1})$
which can be seen as the wasted information of $Y$, because 1) the evaluation values in the last iteration $Y_g$ cannot be utilized and 2) the information of previous evaluation values $\overline{Y}_{g-1}$ is fully utilized only if $H(\overline{Y}_{g-1}|Z,\overline{X}_{g-1})=0$, i.e., $\overline{Y}_{g-1}$ can be reconstructed with $Z$ given $\overline{X}_{g-1}$.

%Note that only the information about the objective function is considered in this definition. The utilization of other information, like the positions of the solutions $x$ in the search space, is irrelevant here. Algorithms could update sample solutions even without any evaluation, which would be meaningless. Since the task is to find the optimal point(s) of the objective function, the use of the objective function's information is the most important.

\section{IURs of Heuristic Optimization Algorithms}\label{sec3}
In order to calculate the IURs of algorithms, we further assume $f(x) \in \mathcal{Y}$ is identically and independently distributed (i.i.d).
%These are standard assumptions for theoretical analyses.
%In discrete optimization problems, $\mathcal{Y}$ is usually a finite set. For example, if there are 100 cities in a travelling salesman problem \cite{lawler1985traveling}, $n$ can be up to 100!. In real-world optimization problems, there is always a precision in the evaluation value. So the codomain can always be seen as discrete. Even if the codomain $\mathcal{Y}$ is infinite or uncountable, the calculated results can still reflect the extent of information utilization because $|\mathcal{Y}|$ only influences the denominator. When several algorithms are used to solve the same kind of optimization algorithms, the denominator is not important because it is constant. The second assumption is natural when there is no prior distribution of the objective functions. However, it is also an extreme case where no information is useful. So it can only be used to calculate IURs. It cannot be used to compare the performances of algorithms \cite{wolpert1997no}. The uniform distribution can be easily replaced with any other ones (such as the Gaussian distribution), because it only influences the denominator (by a multiplier).
In most cases, it is unwise to calculate the IUR by definition. To calculate the denominator is quite straightforward under the above assumption, which equals the number of evaluations times $H(f(x))$. For example, if there are 100 cities in a travelling salesman problem \cite{lawler1985traveling} and $f(x)$ obey uniform distribution, then $|\mathcal{Y}| = 100!, H(f(x))=\log 100!$.
While on the other hand, to directly calculate the numerator is difficult and unnecessary. In each iteration, the output $\mathscr{A_i}(D_{i-1})$ is a certain distribution, which is usually determined by some parameters in the algorithm. In fact we can certainly find (or construct) the set of intermediate parameters $M_i$ such that 1) there is a bijection from $M_i$ to $Z_i$ given $\overline{X}_{i-1}$ and 2) $M_i$ is determined only by $\overline{Y}_{i-1}$ (otherwise $H(Z_i|\overline{X}_{i-1},\overline{Y}_{i-1})>0$), then
\begin{equation}
\sum_{i = 2}^g {H(Z_i|\overline{X}_{i-1},\overline{Z}_{i-1})} = \sum_{i = 2}^g {H(M_i|\overline{M}_{i-1})}=H(M).
\end{equation}
We only have to know the information quantity that is required to determine these intermediate parameters.

In the following, we investigate the IURs of several heuristics to show the procedure of calculating the IUR. Although these algorithms are designed for continuous (domain) optimization, the IURs of any kind of (discrete, combinatorial, dynamic, multi-objective) optimization algorithms can be calculated in the same way as long as there are a domain and a codomain. Without loss of generality, the following algorithms are all minimization algorithms, that is, they all intend to find the solution with the minimal evaluation value in the search space.

\subsection{Random Search Algorithms}
\subsubsection{Monte Carlo}
The Monte Carlo (MC) method is often considered as a baseline for optimization algorithms. It is not a heuristic algorithm and usually fails to find acceptable solutions. If the maximal evaluation number is $m$, MC just uniformly randomly sample $m$ solutions from $\mathcal X$.

MC does not utilize any information about the objective function because $Z$ is fixed.

\begin{proposition}
    \begin{equation}
    \text{IUR}_\text{MC} = 0.
    \end{equation}
\end{proposition}

\subsubsection{Luus-Jaakola}
Luus-Jaakola (LJ) \cite{luus1973optimization} is a heuristic algorithm based on MC. In each iteration, the algorithm generates a new individual $y$ with the uniform distribution within a hypercube whose center is the position of the current individual $x$. If $f(y)<f(x)$, $x$ is replaced by $y$; otherwise, the radius of the hypercube is multiplied by a parameter $\gamma<1$.

The output of LJ in each iteration is the uniform distribution within the hypercube, which is determined by the position $x$ and the radius. They are both controlled by the comparison result, i.e., $I(f(y)<f(x))$. $f(y)$ is i.i.d, but $f(x)$ is the best in the history. Thus,
$H(M_i|\overline{M}_{i-1}) = H(I(f(y)<f(x))|\overline{M}_{i-1}) = \pi(i-1)$.

\begin{proposition}
    \begin{equation}
    \text{IUR}_\text{LJ}(g) = \frac{\sum_{i=1}^{g-1}\pi(i)}{g H(f(x))}.
    \end{equation}
\end{proposition}
%For example, if $n = 1024, g =10$, then $\textit{IUR}_\textit{LJ}= 0.0621$, which is independent of the base of the logarithm (similarly hereinafter).

\subsection{Evolution Strategies}
\subsubsection{$(\mu,\lambda)$-Evolution Strategy}
$(\mu,\lambda)$-evolution strategy (ES) \cite{back1991survey} is an important heuristic algorithm in the family of evolution strategies. In each generation, $\lambda$ new offspring are generated from $\mu$ parents by crossover and mutation with normal distribution, and then the parents of a new generation are selected from these $\lambda$ offspring. As a self-adaptive algorithm, the step size of the mutation is itself mutated along with the position of an individual.

The distribution for generating new offspring is determined by the $\mu$ parents, namely the indexes of the best $\mu$ of the $\lambda$ individuals. Each set of $\mu$ candidates has the same probability to be the best. $H(M_i|\overline{M}_{i-1}) = H(M_i) = \log \binom{\lambda}{\mu}$, where $\binom{\lambda}{\mu} = \frac{\lambda!}{\mu!(\lambda-\mu)!}$.

\begin{proposition}
\begin{equation}
\text{IUR}_\text{$(\mu,\lambda$)-ES}(g) = \frac{(g-1)\log \binom{\lambda}{\mu}}{g\lambda H(f(x))}.
\label{eq11}
\end{equation}
\end{proposition}
%For example, if $n = 1024, g=10, \lambda =10, \mu = 5$, then $\textit{IUR}_\textit{$(\mu,\lambda$)-ES}= 0.0718$.

\subsubsection{Covariance Matrix Adaptation Evolution Strategy}
In order to more adaptively control the mutation parameters in $(\mu,\lambda)$-ES, a covariance matrix adaptation evolution strategy (CMA-ES) was proposed \cite{hansen2005cma}. CMA-ES is a very complicated estimation of distribution algorithm \cite{larranaga2002review}, which adopts several different mechanisms to adapt the mean, the covariance matrix and the step size of the mutation operation. It is very efficient on benchmark functions %\cite{auger2005restart,loshchilov2013cma}
especially when restart mechanisms are adopted. %It won the first place in CEC 2005 \cite{hansen2006compilation} and CEC 2013 \cite{loshchilov2013ranking} competitions.
CMA-ES cannot be introduced here in detail. We refer interested readers to an elementary tutorial: \cite{hansen2005cma}.

Given $\overline{X}_{i-1}$, the mean, the covariance matrix and the step size of the distribution is determined by the indexes and the rankings of the best $\mu$ individuals in each iteration in history. $H(M_i|\overline{M}_{i-1}) = \log \frac{{\lambda !}}{{(\lambda  - \mu )!}}$.

\begin{proposition}
\begin{equation}
\text{IUR}_\text{CMA-ES}(g) = \frac{(g-1)\log \frac{{\lambda !}}{{(\lambda  - \mu )!}}}{g\lambda H(f(x))}.
\label{eq12}
\end{equation}

\end{proposition}
%For example, if $n =1024, g=10, \lambda = 10, \mu = 5$, then $\textit{IUR}_\textit{CMA-ES} =  0.1340$.

Compared with $(\mu,\lambda)$-ES, it is obvious that $\textit{IUR}_\textit{CMA-ES}\geq\textit{IUR}_\textit{$(\mu,\lambda)$-ES}$, because not only the indexes of the $\mu$ best individuals, but also their rankings are used in CMA-ES (to calculate their weights, for example). By utilizing the information of the solutions more thoroughly, CMA-ES is able to obtain more accurate knowledge of the objective function and search more efficiently.

The IURs of Particle Swarm algorithms \cite{eberhart1995new,bratton2007defining} and Differential Evolution algorithms \cite{storn1997differential,zhang2009jade} are also investigated, shown in the appendix. If readers are interested in the IURs of other algorithms, we encourage you to conduct an investigation on your own which can be usually done with limited effort.

\section{IUR versus Performance}\label{sec4}
The IUR is an intrinsic property of a heuristic algorithm, but the performance is not. Besides the algorithm itself, the performance of a heuristic also depends on the termination criterion, the way to measure the performance, and most importantly the distribution of the objective functions. A well-known fact about  performance is that no algorithm outperforms another when there is no prior distribution \cite{wolpert1997no}, which is quite counter-experience. The objective functions in the real world usually subject to a certain underlying distribution. Although it is usually very difficult to precisely describe this distribution, we know that it has a much smaller information entropy than the uniform distribution and hence there is a free lunch \cite{Streeter2003Two,auger2010continuous,everitt2014free}. In this case, the objective function (and resultantly its optimal point) can be identified with limited information quantity (the entropy of the distribution).

Reconsider the setting of the no free lunch (NFL) theorem from the perspective of information utilization. Assume $|\mathcal{X}|=m$ and $|\mathcal{Y}|=n$. Under the setting of NFL (no prior distribution), the total uncertainty of the objective function is $\log n^m = m \log n$. In each evaluation, the information acquired is $\log n$. Therefore, no algorithm is able to certainly find the optimal point of the objective function within less than $m$ times of evaluation even if all acquired information is thoroughly utilized. In this case, enumeration is the best algorithm \cite{English1999Some}. On the contrary, if we already know the objective function is a sphere function, which is determined only by its center, then the required information quantity is $\log m$, and the least required number of evaluations is (more than) ${\log m}/{\log n}=\log_n m$. Suppose the dimensionality of the search space is $d$, then $n$ is $O(m^{\frac{1}{d}})$, $\log_n m$ is $O(d)$, which is usually acceptable. If information is fully utilized ($\textit{IUR}\approx1$), the exact number is $d+1$ \cite{auger2010continuous}. While for algorithms with smaller IURs, more evaluations are needed. For example, if the IUR of another algorithm is half of the best algorithm (with half of the acquired information wasted), then at least about $2d+2$ evaluations are needed.

How much information is utilized by the algorithm per each evaluation determines the lower bound of the required evaluation number to locate the optimal point. \textbf{In this sense, IUR determines the upper bound of an algorithm's performance.} That is, algorithms with larger IURs have greater potential. However, the actual performance also depends on the manner of information utilization and how it accords with the underlying distribution of the objective function. For instance, one can easily design an algorithm with the same IUR as CMA-ES but does not work.

In the following, we will give empirical evidences on the correlation between IUR and performance. The preconditions of the experiments include: 1) the algorithms we investigate here are reasonably designed to optimize the objective functions from the underlying distribution; 2) the benchmark suite is large and comprehensive enough to represent the underlying distribution. The following conclusions may not hold for algorithms that are not reasonably designed or for a narrow or special range of objective functions. In other words, if the manner of information utilization does not accord with the underlying distribution of objective functions, utilizing more information is not necessarily advantageous.

The theoretical correctness of IUR does not rely on these experimental results, but these examples may help readers understand how and to what extent IUR influences performance.
\subsection{Different Parameters of the Same Algorithm}
Sometimes for a certain optimization algorithm the IUR is influenced by only a few parameters. For these algorithms, we may adapt these parameters to show the correlation between the tendency of IUR and the tendency of performance.
%An interesting application of the IUR is to guide the choice of the parameters.

\subsubsection{$(\mu,\lambda)$-ES}
Intuitively, using $\mu = \lambda$ is not a sensible option for $(\mu,\lambda)$-ES (commonly used $\mu / \lambda$ values are in the range from 1/7 to 1/2 \cite{Beyer:2007}) because it makes the selection operation invalid. Now we have a clearer explanation: the IUR of $(\mu,\lambda)$-ES is zero if $\mu = \lambda$ (see Eq. \eqref{eq11}), i.e., $(\mu,\lambda)$-ES does not use any heuristic information if $\mu = \lambda$.

Using a $\mu$ around $\frac{1}{2}\lambda$ may be a good choice for $(\mu,\lambda)$-ES because it leads to a large IUR. When $\mu=\frac{1}{2}\lambda$, the information used by $(\mu,\lambda)$-ES is the most. From the perspective of exploration and exploitation, we may come to a similar conclusion. If $\mu$ is too small (elitism), the information of the population is only used to select the best few solutions, and resultantly the diversity of the population may suffer quickly. If $\mu$ is too large (populism), the information of the population is only used to eliminate the worst few solutions, and resultantly the convergence speed may be too slow.

Different values of $\mu/\lambda$ are evaluated on the CEC 2013 benchmark suite containing 28 different test functions (see Table \ref{tab:cec13}) which are considered as black-box problems \cite{liang2013problem}. The meta parameter is set to $\Delta \sigma = 0.5$. The algorithm using each set of parameters is run 20 times independently for each function. The dimensionality is $d=5$, and the maximal number of function evaluations is $10000d$ for each run. For each fixed $\lambda$, the mean errors of 20 independent runs of each $\mu/\lambda$ are ranked. The rankings are averaged over 28 functions, shown in Fig. \ref{fig2}. $-\log \binom{\lambda}{\mu}/\lambda$ curves are translated along the vertical axis, also shown in Fig. \ref{fig2}.

\begin{table}[htbp]
    \scriptsize
  \centering
  \caption{Test functions of CEC 2013 single objective optimization benchmark suite \cite{liang2013problem}}
    \begin{tabular}{|c|c|l|}
    \hline
          & No.   & Name \\
    \hline
    \multicolumn{1}{|c|}{\multirow{5}[0]{2cm}{Unimodal Functions}} & 1     & Sphere Function \\
    \multicolumn{1}{|c|}{} & 2     & Rotated High Conditioned Elliptic Function \\
    \multicolumn{1}{|c|}{} & 3     & Rotated Bent Cigar Function \\
    \multicolumn{1}{|c|}{} & 4     & Rotated Discus Function \\
    \multicolumn{1}{|c|}{} & 5     & Different Powers Function \\ \hline
    \multicolumn{1}{|c|}{\multirow{15}[0]{2cm}{Basic Multimodal Functions}} & 6     & Rotated Rosenbrock’s Function \\
    \multicolumn{1}{|c|}{} & 7     & Rotated Schaffers F7 Function \\
    \multicolumn{1}{|c|}{} & 8     & Rotated Ackley’s Function \\
    \multicolumn{1}{|c|}{} & 9     & Rotated Weierstrass Function \\
    \multicolumn{1}{|c|}{} & 10    & Rotated Griewank’s Function \\
    \multicolumn{1}{|c|}{} & 11    & Rastrigin’s Function \\
    \multicolumn{1}{|c|}{} & 12    & Rotated Rastrigin’s Function \\
    \multicolumn{1}{|c|}{} & 13    & Non-Continuous Rotated Rastrigin’s Function \\
    \multicolumn{1}{|c|}{} & 14    & Schwefel's Function \\
    \multicolumn{1}{|c|}{} & 15    & Rotated Schwefel's Function \\
    \multicolumn{1}{|c|}{} & 16    & Rotated Katsuura Function \\
    \multicolumn{1}{|c|}{} & 17    & Lunacek Bi\_Rastrigin Function \\
    \multicolumn{1}{|c|}{} & 18    & Rotated Lunacek Bi\_Rastrigin Function \\
    \multicolumn{1}{|c|}{} & 19    & Expanded Griewank’s plus Rosenbrock’s Function \\
    \multicolumn{1}{|c|}{} & 20    & Expanded Scaffer’s F6 Function \\ \hline
    \multicolumn{1}{|c|}{\multirow{8}[0]{2cm}{Composition Functions}} & 21    & Composition Function 1 (Rotated) \\
    \multicolumn{1}{|c|}{} & 22    & Composition Function 2 (Unrotated) \\
    \multicolumn{1}{|c|}{} & 23    & Composition Function 3 (Rotated) \\
    \multicolumn{1}{|c|}{} & 24    & Composition Function 4 (Rotated) \\
    \multicolumn{1}{|c|}{} & 25    & Composition Function 5 (Rotated) \\
    \multicolumn{1}{|c|}{} & 26    & Composition Function 6 (Rotated) \\
    \multicolumn{1}{|c|}{} & 27    & Composition Function 7 (Rotated) \\
    \multicolumn{1}{|c|}{} & 28    & Composition Function 8 (Rotated) \\
    \hline
    \end{tabular}%
  \label{tab:cec13}%
\end{table}%

\begin{figure}[htb]
\subfigure[$\lambda = 10$] {
\includegraphics[width=0.5\columnwidth]{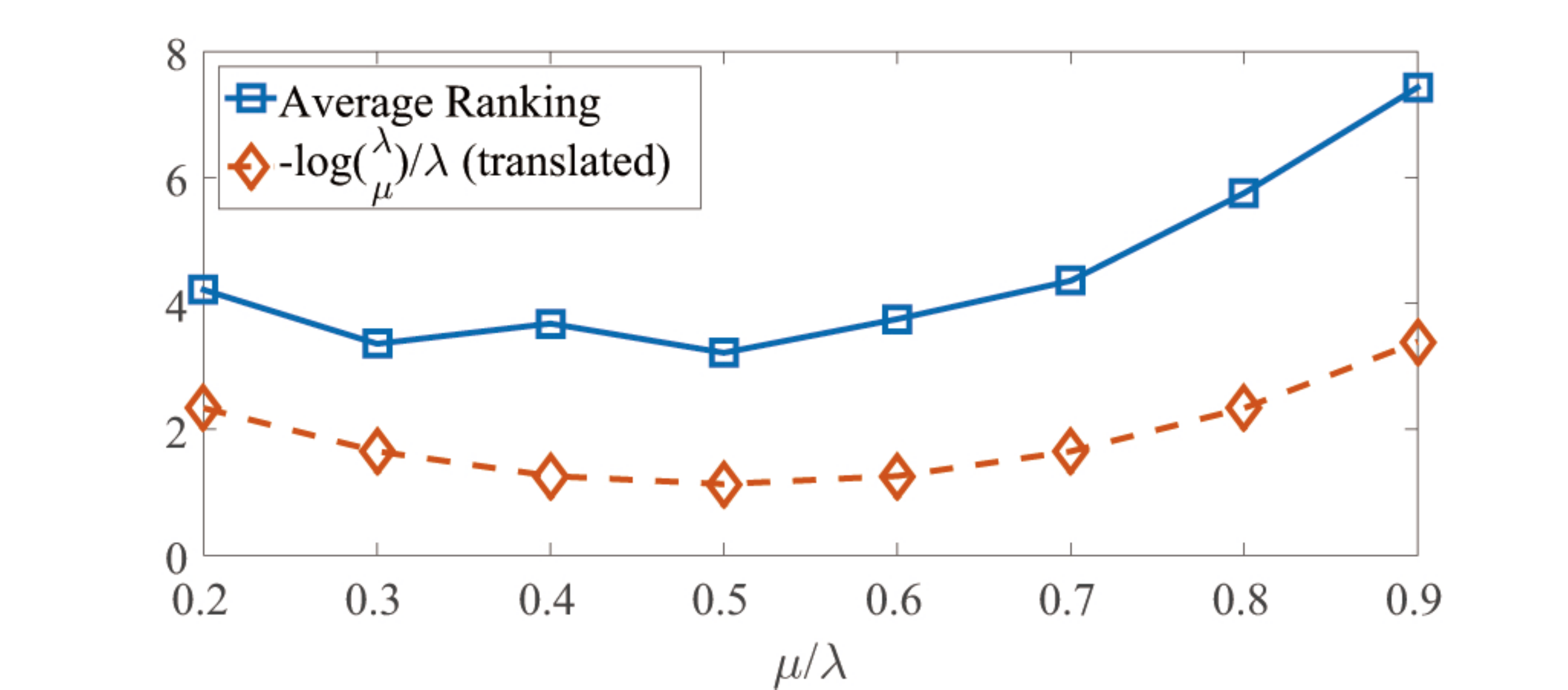}
}%
\subfigure[$\lambda = 20$] {
\includegraphics[width=0.5\columnwidth]{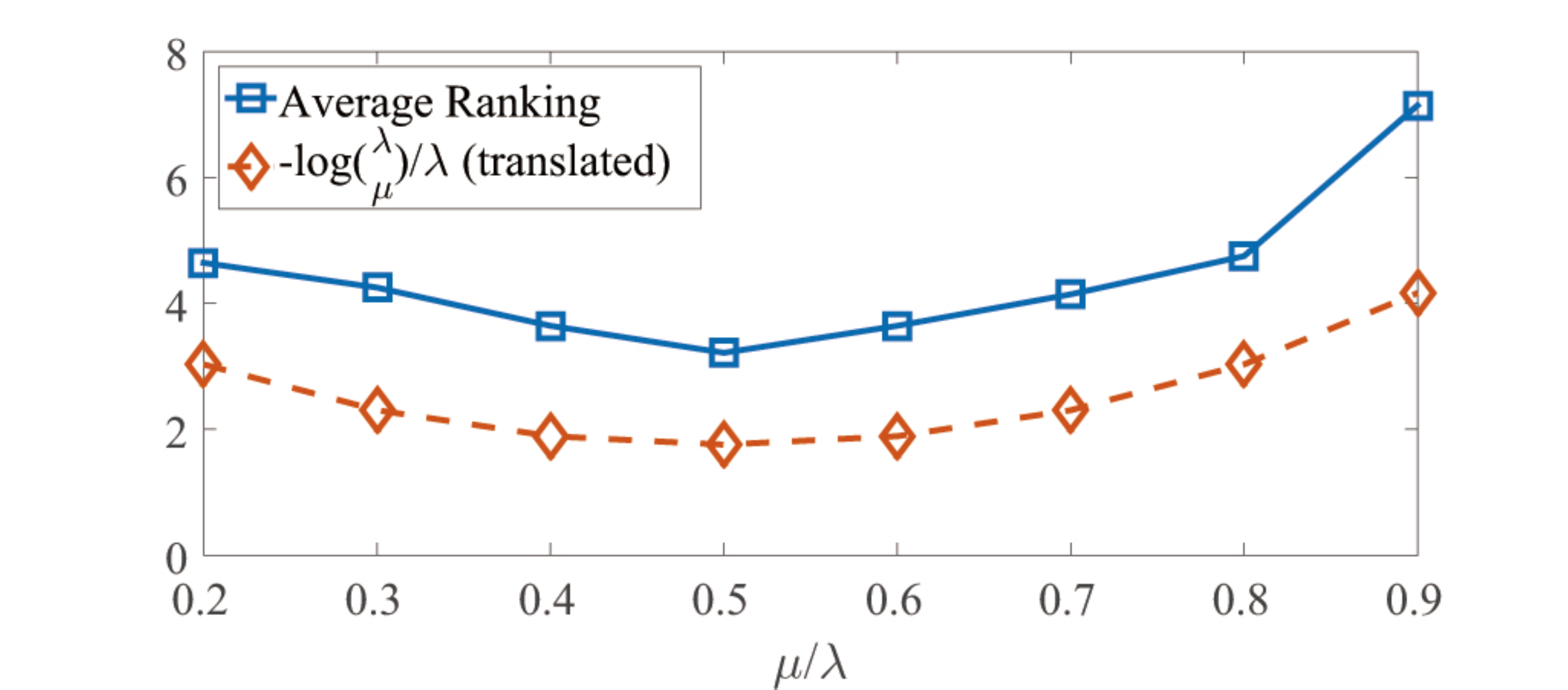}
}%
\\
\subfigure[$\lambda = 30$] {
\includegraphics[width=0.5\columnwidth]{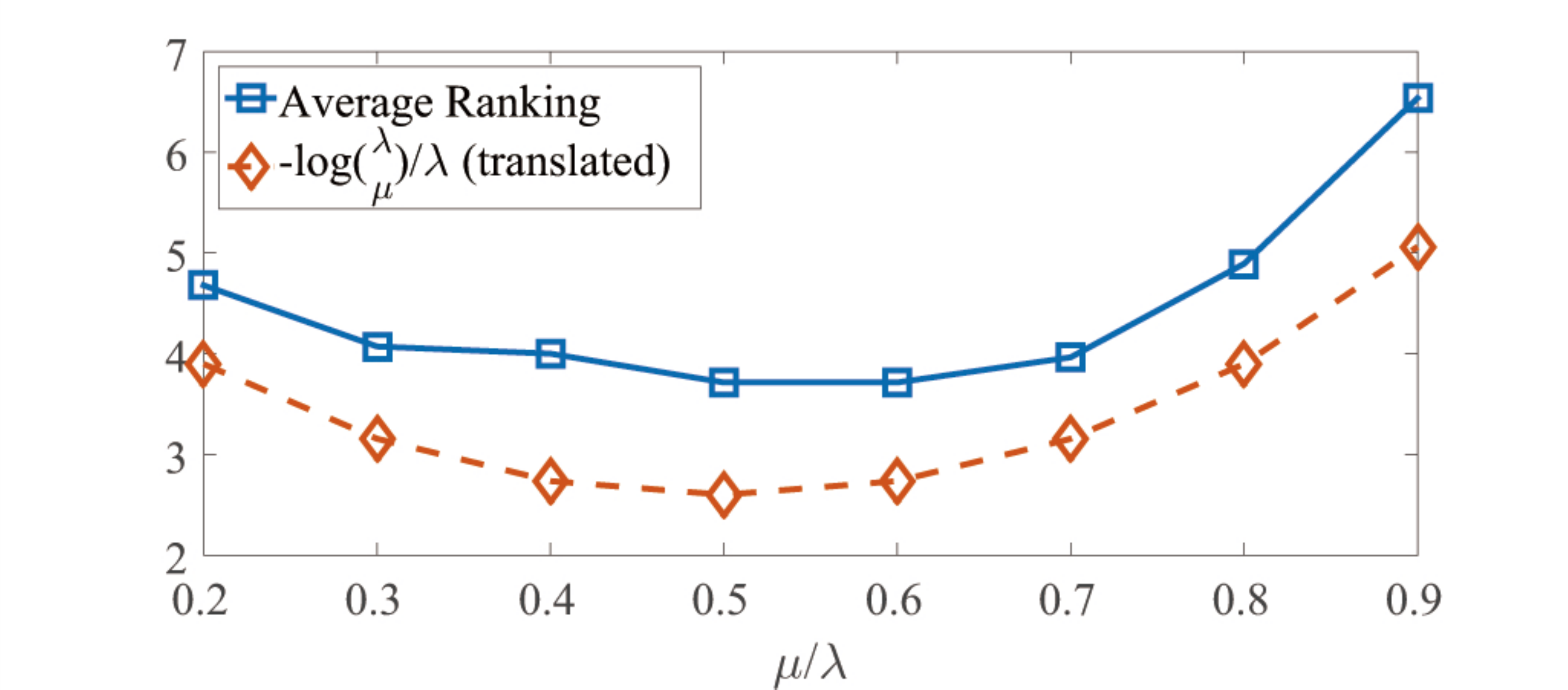}
}%
\subfigure[$\lambda = 40$] {
\includegraphics[width=0.5\columnwidth]{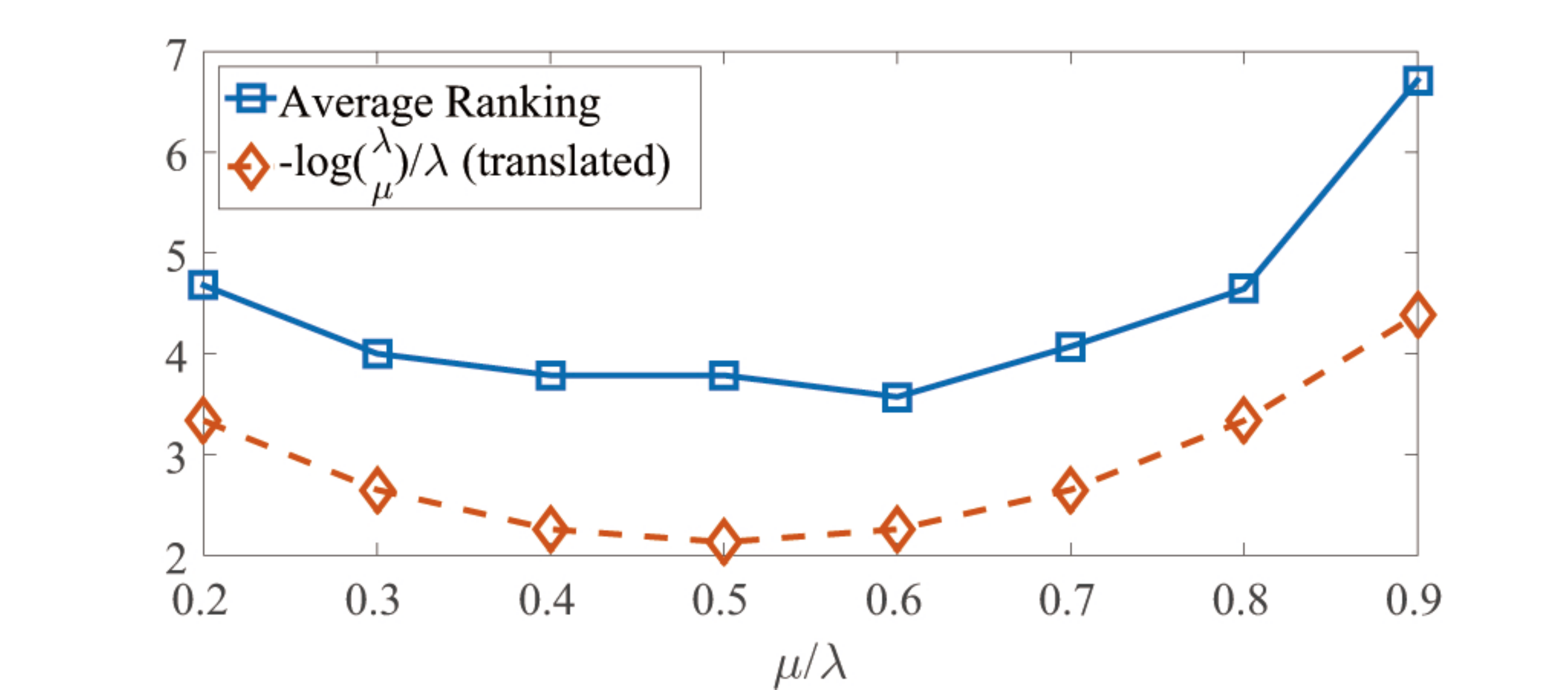}
}%
%\subfigure{
%\includegraphics[width=.5\textwidth]{fig1.pdf}
%\caption{$\lambda = 10$}
%}%
%\subfigure{
%\centering
%\includegraphics[width=.5\textwidth]{fig2.pdf}
%\caption{$\lambda = 20$}
%}
%
%\subfigure{
%\centering
%\includegraphics[width=.5\textwidth]{fig3.pdf}
%\caption{$\lambda = 30$}
%}%
%\subfigure{
%\centering
%\includegraphics[width=.5\textwidth]{fig4.pdf}
%\caption{$\lambda = 40$}
%}
\caption{The average rankings and the $-\log \binom{\lambda}{\mu}/\lambda$ curves of each value of $\lambda$.}
\label{fig2}
\end{figure}

According to the experimental results, $\mu/\lambda$ around 0.5 is a good choice, which accord with our expectation. Moreover, the tendency of the performance (average ranking) curve is generally identical to that of the IUR ($\textit{IUR}_{(\mu,\lambda)-\textit{ES}} \propto \log \binom{\lambda}{\mu}/\lambda$ when $g$ is large). The experimental results indicate a positive correlation between the performance and the IUR: the parameter value with larger IUR is prone to perform better. Thus, the IUR can be used to guide the choice of parameters. After all, tuning the parameters by experiments is much more expensive than calculating the IURs.

\subsubsection{CMA-ES}
Different from $(\mu,\lambda)$-ES, CMA-ES adopts a rank-based weighted recombination instead of a selection operation, in which the rank information of the best $\mu$ individuals is utilized.

On the one hand, the rank-based weighted recombination achieves the largest IUR when $\mu = \lambda$ (see Eq. \eqref{eq12}). The larger $\mu$ is, the more information is used (because the rank information of the rest $\lambda-\mu$ individuals are wasted).

On the other hand, the rank-based weighted recombination also achieves the best performance when $\mu = \lambda$ \cite{Arnold2005Optimal}. This is also an evidence on the correlation between IUR and performance. However, the optimal weighted recombination requires the use of negative weights, which is somehow not adopted in CMA-ES \cite{hansen2005cma}. The manner of information utilization and other conditions (termination criterion, performance measure, etc.) should also be taken into consideration when parameters are chosen. Therefore using $\mu = \lambda$ is probably not the best choice for CMA-ES even though it leads to a large IUR.

\subsection{Algorithms in the Same Family}
Usually different algorithms in the same family utilize information in similar manners, in which case we may compare their performances to show the correlation between IUR and performance. However, we need to be more cautious here because IUR is not the only factor as long as different algorithms in the same family do not utilize information in identical manners.

As shown in Section 3, $\textit{IUR}_\textit{LJ}\geq\textit{IUR}_\textit{MC}$ and $\textit{IUR}_\textit{CMA-ES}\geq\textit{IUR}_\textit{$(\mu,\lambda)$-ES}$. LJ and CMA-ES are more finely designed compared with the previous algorithms since they are able to utilize more information of the objective function. Naturally we would expect that LJ outperforms MC and CMA-ES outperforms $(\mu,\lambda)$-ES.

The four algorithms are evaluated on the CEC 2013 benchmark suite. The parameter of LJ is set to $\gamma = 0.99$. The parameters of $(\mu,\lambda)$-ES are set to $\lambda = 30, \mu =15, \Delta \sigma=0.5$. The parameters of CMA-ES are set to suggested values \cite{hansen2005cma} except that $\sigma = 50$ because the radius of the search space is 100. The dimensionality is $d=5$, and the maximal number of function evaluation is $10000d$ for each run. Each algorithm is run 20 times independently for each function. Their mean errors are shown in Table \ref{tab:addlabel}. The best mean errors are highlighted. Their mean errors are ranked on each function, and the average rankings (AR.) over 28 functions are also shown in Table \ref{tab:addlabel}.

% Table generated by Excel2LaTeX from sheet 'Sheet2'
\begin{table}[htbp]
\scriptsize
  \centering
  \caption{Mean errors and average rankings of the four algorithms and $p$ values}

    \begin{tabular}{|c|cccc|cc|}
    \hline
    \multirow{2}*{F.}    & \multirow{2}*{MC}    & \multirow{2}*{LJ}    & \multirow{2}*{$(\mu,\lambda)$-ES}    & \multirow{2}*{CMA-ES}  & \multirow{2}*{MC vs. LJ} & $(\mu,\lambda)$-ES \\
          &       &       &                       &        &           &    vs. CMA-ES     \\
    \hline
    1     & 2.18E+02 & \textbf{0.00E+00} & 5.91E-12 & \textbf{0.00E+00} & \underline{8.01E-09} & \underline{4.01E-02} \\
    2     & 4.25E+05 & \textbf{0.00E+00} & 3.49E+05 & \textbf{0.00E+00} & \underline{8.01E-09} & \underline{8.01E-09} \\
    3     & 8.02E+07 & \textbf{0.00E+00} & 2.18E+07 & \textbf{0.00E+00} & \underline{1.13E-08} & \underline{1.13E-08} \\
    4     & 4.23E+03 & \textbf{0.00E+00} & 2.20E+04 & \textbf{0.00E+00} & \underline{1.13E-08} & \underline{1.13E-08} \\
    5     & 8.00E+01 & 6.79E+01 & 1.95E-05 & \textbf{0.00E+00} & \underline{1.33E-02} & \underline{1.90E-04} \\
    6     & 9.94E+00 & 2.51E+01 & \textbf{2.46E+00} & 7.86E-01 & \underline{4.17E-05} & \underline{8.15E-06} \\
    7     & 2.02E+01 & 7.10E+01 & 1.66E+01 & \textbf{5.66E+00} & 1.99E-01 & \underline{2.56E-03} \\
    8     & \textbf{1.83E+01} & 2.01E+01 & 2.03E+01 & 2.10E+01 & \underline{3.42E-07} & \underline{1.61E-04} \\
    9     & 2.53E+00 & 1.67E+00 & 2.37E+00 & \textbf{1.08E+00} & \underline{1.48E-03} & \underline{1.63E-03} \\
    10    & 2.30E+01 & 1.78E+00 & 1.30E+01 & \textbf{4.16E-02} & \underline{6.80E-08} & \underline{1.23E-07} \\
    11    & 2.22E+01 & 1.40E+01 & 6.67E+00 & \textbf{6.57E+00} & \underline{3.04E-04} & 8.17E-01 \\
    12    & 2.10E+01 & 1.33E+01 & 1.20E+01 & \textbf{7.36E+00} & \underline{1.12E-03} & \underline{2.04E-02} \\
    13    & 2.21E+01 & 1.90E+01 & 1.87E+01 & \textbf{1.28E+01} & 1.20E-01 & 5.98E-01 \\
    14    & 3.78E+02 & 7.53E+02 & \textbf{1.35E+02} & 4.61E+02 & \underline{1.10E-05} & \underline{7.41E-05} \\
    15    & \textbf{3.84E+02} & 6.85E+02 & 5.27E+02 & 4.52E+02 & \underline{3.99E-06} & 1.81E-01 \\
    16    & 7.43E-01 & \textbf{5.34E-01} & 8.27E-01 & 1.49E+00 & \underline{1.93E-02} & 1.11E-01 \\
    17    & 3.25E+01 & 2.23E+01 & \textbf{9.87E+00} & 1.07E+01 & \underline{1.78E-03} & 3.65E-01 \\
    18    & 3.43E+01 & 1.82E+01 & \textbf{1.01E+01} & 1.01E+01 & \underline{2.60E-05} & 9.89E-01 \\
    19    & 4.08E+00 & 7.21E-01 & 5.45E-01 & \textbf{4.82E-01} & \underline{9.17E-08} & 9.46E-01 \\
    20    & \textbf{1.23E+00} & 1.85E+00 & 2.50E+00 & 1.92E+00 & \underline{1.10E-05} & \underline{6.97E-06} \\
    21    & 3.23E+02 & 3.05E+02 & \textbf{2.55E+02} & 2.80E+02 & \underline{1.94E-02} & 9.89E-01 \\
    22    & 5.91E+02 & 7.91E+02 & \textbf{4.01E+02} & 7.20E+02 & \underline{2.56E-03} & \underline{5.63E-04} \\
    23    & \textbf{6.04E+02} & 8.33E+02 & 7.01E+02 & 6.08E+02 & \underline{8.29E-05} & 3.37E-01 \\
    24    & \textbf{1.26E+02} & 2.04E+02 & 1.99E+02 & 1.76E+02 & \underline{6.80E-08} & \underline{4.60E-04} \\
    25    & \textbf{1.27E+02} & 1.96E+02 & 1.98E+02 & 1.81E+02 & \underline{1.60E-05} & \underline{7.71E-03} \\
    26    & \textbf{1.01E+02} & 2.38E+02 & 1.67E+02 & 1.98E+02 & \underline{1.43E-07} & 7.76E-01 \\
    27    & 3.57E+02 & 3.52E+02 & 3.65E+02 & \textbf{3.27E+02} & 4.25E-01 & \underline{2.47E-04} \\
    28    & 3.05E+02 & \textbf{3.00E+02} & 3.25E+02 & 3.15E+02 & 8.59E-01 & 2.03E-01 \\
    \hline
    AR.   & 2.82  & 2.68  & 2.43  & \textbf{1.82 } &  14 : 10     & 14 : 3 \\
    \hline
    \end{tabular}%

  \label{tab:addlabel}%
\end{table}%

Pair-wise Wilcoxon rank sum tests are also conducted between MC and LJ and between $(\mu,\lambda)$-ES and CMA-ES. The $p$ values are shown in the last two columns of Table \ref{tab:addlabel}. Significant results (with confidence level 95\%) are underlined. The results of LJ are significantly better than MC on 14 functions, and significantly worse on only 10 functions. While the results of CMA-ES are significantly better than $(\mu,\lambda)$-ES on 14 functions, and significantly worse on only 3 functions. Generally speaking, the performance of LJ is better than MC and the performance of CMA-ES is better than $(\mu,\lambda)$-ES. These experimental results imply that the extent of information utilization may be an important factor in the performance.

%Again, they do not imply IUR is the only factor.
The algorithms in the same family utilize information in similar but different manners. In this case, the influence of IUR on the performance is crucial, but sometimes not deterministic. For example, in CMA-ES, there are several different mechanisms proposed to improve the performance. The improvement in IUR does not reflect all of them. The improvement related to IUR is the rank-based weighted recombination. It has significant impact on performance \cite{Arnold2005Optimal,Hansen2004Evaluating}. While other mechanisms such as adapting the covariance matrix and the step size are not related to IUR but also very important. These mechanisms are introduced as different information utilization manners, which help the algorithm to better fit the underlying distribution of objective functions. Similar comparisons can be made between PSO and SPSO and between DE and JADE (see appendix).

%In addition, the fact that SPSO outperforms PSO \cite{bratton2007defining} and JADE outperforms DE \cite{zhang2009jade} has already been proven. (It is shown in the supplementary material that $\textit{IUR}_\textit{SPSO}\geq\textit{IUR}_\textit{PSO}$ and $\textit{IUR}_\textit{JADE}\geq\textit{IUR}_\textit{DE}$.)

However, after all, an algorithm cannot perform very well if little information is used. Hence, just like LJ, CMA-ES, SPSO and JADE, the tendency of elevating the IUR is quite clear in various families of heuristics. Many mechanisms have been proposed to better preserve historical information for further utilization \cite{Deb2000A,Reyes2006Multi,zhang2009jade}. Many general methods (adaptive parameter control \cite{eiben1999parameter}, estimation of distribution \cite{larranaga2002review}, fitness approximation \cite{jin2005comprehensive}, Bayesian approaches \cite{pelikan2005bayesian}, Gaussian process models \cite{buche2005accelerating}, hyper-heuristic \cite{burke2010classification}) have been proposed to elevate the IURs of heuristics. Not to mention these numerous specified mechanisms. In summary, the IUR provides an important and sensible perspective on the developments in this field.

%The opposite is rare.

%Explicitly, the adaptations made in LJ, CMA-ES, JADE and SPSO are effective on improving the performance. Implicitly, they improve performance by improving the information utilization. There are many more algorithms that follow the same pattern: the performance improves along with the IUR. Actually, elevating the IUR is the potential driver of many algorithmic improvement works, such as adaptive parameter control \cite{eiben1999parameter}, estimation of distribution \cite{larranaga2002review}, fitness approximation \cite{jin2005comprehensive}, Bayesian approaches \cite{pelikan2005bayesian}, Gaussian process models \cite{buche2005accelerating}, hyper-heuristic \cite{burke2010classification}, etc.

\subsection{Algorithms in Different Families}
The correlation between the IUR and the performance of the algorithms in different families (such as LJ and $(\mu,\lambda)$-ES) can be vaguer because the manners of information utilization are different, though the above experimental results accord with our expectation ($(\mu,\lambda)$-ES performs better than LJ and $\textit{IUR}_\textit{LJ}\leq\textit{IUR}_\textit{$(\mu,\lambda)$-ES}$ unless $\mu=\lambda$). If algorithms utilize information in extremely different manners, the IUR may not be the deterministic factor. %How the information utilization manner accords with the underlying distribution becomes more important here.
There are infinite manners to utilize information. It is difficult to judge which manner is better. Whether a manner is good or not depends on how it fits the underlying distribution of the objective functions, which is difficult to describe. A well designed algorithm with low IUR may outperform a poorly designed algorithm with high IUR because it utilizes information more efficiently and fits the underlying distribution better. Nonetheless, certainly the extent of information utilization is still of importance in this case because 1) the algorithms with larger IURs have greater potential 2)
the IUR of the ``best" algorithm (if any) must be very close to one and 3) an algorithm that uses little information cannot be a good algorithm.

The exact correlation between the IUR and the performance requires much more theoretical works on investigating the manners of information utilization and how they fit the underlying distributions, which are very difficult but not impossible.

\section{Upper Bound for Comparison-based Algorithms}
Above examples have covered several approaches of information utilization in heuristic optimization algorithms. But the IURs of these algorithms are all not high because they are comparison-based algorithms, in which only the rank information is utilized.
\begin{theorem}[Upper bound for comparison-based algorithms]
If the maximal number of evaluations is $m$, $y=f(x)$ are i.i.d, and algorithm $\mathscr{A}$ is a comparison-based optimization algorithm,
\begin{equation}
\text{IUR}_\mathscr{A} \le \frac{\log m}{H(f(x))}.
\end{equation}
\end{theorem}
\begin{proof}
Suppose in a certain run, the actual evaluation number is $m'\le m$. In this case, $M$ is drawn from a set with cardinal number at most $m'!$ (with $m'$ individuals all sorted), then the maximal information quantity is $H(M)\le\log m'!$ for a comparison-based algorithm. Thus $\emph{IUR}_\mathscr{A} \le \frac{\log m'!}{m' H(f(x))}$. Note that the right hand side is a monotonically increasing function of $m'$, and $\frac{\log m!}{mH(f(x))} \le \frac{\log m}{{H(f(x))}}.$
\end{proof}

Suppose $|\mathcal{Y}|=n$ and $f(x)$ obey uniform distribution, than $\frac{\log m}{H(f(x))} = log_n m$. Typically $m<<n$, thus this upper bound is quite low. Most iterative algorithms do not allow the information in past iterations (because it requires a lot of memory space to do so), in which case the upper bound becomes $\frac{\log \lambda}{H(f(x))}$ where $\lambda$ is the evaluation number in each generation. The IUR of CMA-ES is able to approach this bound when $\mu=\lambda$. That is, CMA-ES has almost the largest IUR in comparison-based algorithms without historical information.

There exist algorithms which use exact evaluation values in the searching process, such as genetic algorithm \cite{holland1975adaptation}, ant colony optimization \cite{dorigo1996ant}, estimation of distribution algorithms \cite{larranaga2002review}, invasive weed optimization \cite{mehrabian2006novel}, artificial bee colony \cite{karaboga2007powerful}, fireworks algorithm \cite{tan2010fireworks}, etc. They can achieve higher IURs, even close to 1, because the cardinal number of the set from which $M$ is drawn can be up to $n^m$. These algorithms have greater potential than comparison-based algorithms and can outperform them if well designed.

\section{Conclusion}\label{sec5}
It is natural and often effective to utilize more heuristic information in optimization algorithms, which has been widely realized. However, there was no metric to reflect the extent of information utilization. In this paper, a metric called the information utilization ratio (IUR) is defined as the ratio of the utilized information quantity over the acquired information quantity. IUR can be an index to reflect how finely and advanced an algorithm is designed. IUR proves to be well defined. Several examples are given to demonstrate the procedure of calculating IURs. Generally speaking, the IUR determines the upper bound of the performance of an optimization algorithm. To further indicate the importance of this metric, several experiments are conducted to show the correlation between the IUR and the performance. The experimental results imply that 1) for a certain algorithm, the parameter value with larger IUR has advantage; 2) for algorithms in the same family, the one with larger IUR is prone to be more efficient; 3) for algorithms in different families, the IUR is also an important factor. We also give the IUR's upper bound for comparison-based algorithms. 

The IUR can be used to guide the choice of parameters, guide the design of new algorithms and guide the improvement of existing algorithms. For example, if you are inventing a new algorithm, or adapting an existing one, it is promising to include mechanisms that can enhance the information utilization in your algorithm. If you want to know which one among several algorithms is more likely efficient before you use them, it would be quite informative to compare their IURs to show which one is better designed and has greater potential.

Most works in the field of heuristic search or optimization focus on inventing new mechanisms or tricks, while few have considered the potential driver behind these works. We consider this work as a fundamental theory, which is surprisingly not easy. Hopefully the definition of IUR will lead to a more systematic manner of research about how mechanisms should be designed and how information should be utilized.

Extending this metric to other fields in artificial intelligence such as classification and time series prediction may be an interesting future work.

\newpage
\begin{appendix}
\subsection{Particle Swarm Algorithms}
\subsubsection{Particle Swarm Optimization}
Particle swarm optimization (PSO) \cite{eberhart1995new} is one of the most famous swarm and heuristic algorithms which is quite simple but surprisingly efficient in numerical optimization. In PSO, a fixed number ($s$) of particles moves in the search space to find the optimal solutions. The position of a particle is updated as follows. In generation $g$, for each particle $i$ and each dimension $j$,
\begin{align}
v_{ij}(g+1) \leftarrow &v_{ij}(g) + {\phi _1}{r_{1,{ij}}}(pbest_{ij}(g) - x_{ij}(g))\notag \\
&+ {\phi _2}{r_{2,{ij}}}(gbest_j(g) - x_{ij}(g)),
\end{align}
\begin{equation}
x_{ij}(g+1) \leftarrow x_{ij}(g)+v_{ij}(g+1),
\end{equation}
where $\phi _1$ and $\phi _2$ are constant coefficients, $r_1$ and $r_2$ are random numbers, $pbest$ is the best position in history found by this particle and $gbest$ is the best position found by the entire swarm.

The output distribution in each generation is determined by $I(f(x_i(g))<f(pbest_i(g-1)))$ and $\mathop {\arg \min }_i f(pbest_i(g))$. Although it is difficult to calculate $H(M)$, we have the lower and upper bounds:
\begin{equation}
s\sum_{i=1}^{g-1}\pi(i)\le H(M)\le\sum_{i=2}^{g}H(M_i)\le (g-1)\log s + s\sum_{i=1}^{g-1}\pi(i).
\end{equation}

\begin{proposition}
\begin{equation}
\frac{s\sum_{i=1}^{g-1}{\pi (i)}}{sgH(f(x))} \le \text{IUR}_\text{PSO}(g) \le \frac{(g-1)\log s+ s\sum_{i = 1}^{g-1} {\pi (i)}}{sgH(f(x))}.
\end{equation}
\end{proposition}
%For example, if $n = 1024, s =10, g = 10$, then $0.0621\le\textit{IUR}_\textit{PSO}\le 0.0920$.

\subsubsection{Standard Particle Swarm Optimization}
After years of development, many improvements and variants are proposed for PSO. In order to construct a common ground for further researches, a standard particle swarm optimization (SPSO) was defined \cite{bratton2007defining}. Compared with original PSO, there are two main modifications: the local ring topology and the constricted update rule. The constricted update rule uses a new coefficient derived from $\phi _1$ and $\phi _2$ to constrict the velocity to guarantee convergence. In the local ring topology, the $gbest$ in the velocity update equation is replaced with a $lbest$, which is the best position among this individual and its two neighbourhoods on the ring.

For each group (consisting of three particles), information with quantity at most $\log 3$ is needed to decide $lbest$.
\begin{proposition}
\begin{equation}
\frac{s\sum_{i = 1}^{g-1}\pi (i)}{sgH(f(x))} \le \text{IUR}_\text{SPSO}(g) \le \frac{s(g-1)\log 3 + s\sum_{i = 1}^{g-1}\pi (i)}{sgH(f(x))}.
\end{equation}
\end{proposition}
%For example, if $n = 1024, s =10, g = 10$, then $0.0621\le\emph{IUR}_\emph{SPSO} \le  0.2047$.

Usually $\emph{IUR}_\emph{PSO}\leq\emph{IUR}_\emph{SPSO}$ though their exact values are difficult to derive. It turns out that the information utilization ratio of the local model is larger than the global model because in local topology the particles interact with each other more frequently.

According to experimental results, SPSO significantly outperform PSO on a large range of test functions. \cite{bratton2007defining}

\subsection{Differential Evolution Algorithms}
\subsubsection{Differential Evolution}
Differential evolution (DE) \cite{storn1997differential} is a powerful heuristic algorithm for numerical optimization. The number of individuals in DE is also fixed. The mutation is conducted as below (take DE/rand/1 as an example). For each $x$ in the population, generate
\begin{equation}
z = {x_{r1}} + F({x_{r2}} - {x_{r3}}),
\end{equation}
where $r1,r2$ and $r3$ are random indexes and $F$ is a constant coefficient. Then a crossover is conducted between $z$ and $x$ to generate a new candidate $y$, where there is a parameter $CR$ to control the probability that a dimension of $y$ is identical to that of $z$. If $f(y)<f(x)$, $x$ is replaced with $y$, otherwise, $x$ is kept.

In DE, the distribution of generating new offspring is determined by $I(f(y)<f(x))$ of each individual. So the IUR of DE is equal to that of LJ with the same $g$. However, they would be different with the same number of evaluation times.

\begin{proposition}
\begin{equation}
\text{IUR}_\text{DE}(g) = \frac{s\sum_{i = 1}^{g-1}\pi (i)}{sgH(f(x))}.
\end{equation}
\end{proposition}

%For example, if $n = 1024, s =10, g = 10$, then $\emph{IUR}_\emph{DE} =  0.0621$.

IURs of some other DE variants are given in Table \ref{tab:de}.

\begin{table}[htb]
  \centering
  \caption{IURs of other DE variants}
    \begin{tabular}{|c|c|}
    \hline
          & IUR   \\
          \hline
    DE/best/1   &  $=\emph{IUR}_\emph{PSO}$   \\
    \hline
    DE/current-to-best/1 & $=\emph{IUR}_\emph{PSO}$ \\
    \hline
    DE/rand/2 &  $=\emph{IUR}_\emph{DE}$ \\
    \hline
    DE/best/2 & $=\emph{IUR}_\emph{PSO}$ \\
    \hline
    \end{tabular}%
  \label{tab:de}%
\end{table}%

\subsubsection{JADE}
JADE \cite{zhang2009jade} is an important development of DE. There are three main adaptations proposed in JADE:
\begin{enumerate}
\item A DE/current-to-$p$best/1 mutation strategy. In JADE,
\begin{equation}
z_i = x_i + F_i(x^p_{best}-x_i) + F_i({x_{r1}} - {x_{r2}}),
\end{equation}
where $x^p_{best}$ is a randomly chosen individual from the $100p\%$ best individuals.
\item An optional external archive. %In DE, worse candidates in the comparisons are just abandoned. JADE adopts an external archive to store some of them. The $x_{r2}$ in the above equation can be chosen either from the current population or from the external archive.
\item Adaptive mutation parameters. %The means of successful parameters $CR$ and $F$ are recorded to direct future ones. Each individual has a unique $F_i$ and $CR_i$ in JADE.
\end{enumerate}

External archive is a useful tool to improve information utilization. However, in JADE these individuals are just randomly chosen and randomly removed from the archive, where no information of the objective function is used. Compared to DE, JADE elevates IUR after all because the indexes of the best $100p\%$ individuals are used. Note that the output distribution is determined only when all indexes of the best $100p\%$ individuals are given.
\begin{proposition}
\begin{equation}
\frac{s\sum_{i = 1}^{g-1}\pi (i)}{sgH(f(x))} \le \text{IUR}_\text{JADE}(g) \le \frac{(g-1)\log \binom{s}{ps}+s\sum_{i = 1}^{g-1} {\pi (i)} }{sgH(f(x))}.
\end{equation}
\end{proposition}

According to experimental results, JADE significantly outperform DE on a large range of test functions. \cite{zhang2009jade}
\end{appendix}
%For example, if $n = 1024, s =10, g = 10, p = 0.2$, then $0.0621\le\textit{IUR}_\textit{JADE} \le 0.1115$.

%\section*{References}
%\newpage
\bibliographystyle{unsrt}
\bibliography{AAAI17}
\end{document}